%% file: main.tex
\documentclass[10pt,a4paper]{article}
\usepackage[utf8]{inputenc}
\usepackage[sort&compress,numbers]{natbib}
\input{preamble}

\usepackage{nsl_villa}

\title{Noise Space Learning for VLA Adaptation}
\author{\mbox{Junjie Lu\textsuperscript{1,*,\ensuremath{\ddagger}}},
  \mbox{Chuheng Zhang\textsuperscript{2,*}},
  \mbox{Yucen Wang\textsuperscript{3}},
  \mbox{Yuhua Jiang\textsuperscript{4}},
  \mbox{Xinyao Qin\textsuperscript{4}},
  \mbox{Min Xu\textsuperscript{1,\ensuremath{\dagger}}},
  \mbox{Li Zhao\textsuperscript{2,\S}}}
\newcommand{\nslaffiliations}{\textsuperscript{1}University of Technology Sydney\quad
  \textsuperscript{2}Microsoft Research\quad
  \textsuperscript{3}Nanjing University\quad
  \textsuperscript{4}Tsinghua University}
\newcommand{\nslauthornotes}{\textsuperscript{*}Equal contribution.\quad
  \textsuperscript{\ensuremath{\dagger}}Corresponding author: \href{mailto:min.xu@uts.edu.au}{min.xu@uts.edu.au}.\quad
  \textsuperscript{\S}Project lead.\quad
  \textsuperscript{\ensuremath{\ddagger}}Work conducted at MSRA.}
\date{}

\begin{document}
\thispagestyle{nslfirst}
\begin{nsltitlebox}
  \maketitle
  \input{sections/abstract}
\end{nsltitlebox}

\input{sections/introduction}
\input{sections/related_work}

\input{sections/methods}
\input{sections/experiments}
\input{sections/discussion}
\label{sec:main_end}

\clearpage
\color{black}

\clearpage
\appendix
\setlength{\textfloatsep}{14pt plus 2pt minus 2pt}
\setlength{\floatsep}{10pt plus 2pt minus 2pt}
\setlength{\intextsep}{12pt plus 2pt minus 2pt}
\captionsetup{skip=7pt}
\raggedbottom
\input{appendices/information}
\input{appendices/protocol}
\input{appendices/robotwin}
\end{document}

%% file: preamble.tex
\usepackage[T1]{fontenc}
\usepackage{amsmath,amssymb,amsthm}
\usepackage{booktabs,tabularx,multirow,array}
\usepackage{xcolor}
\usepackage{graphicx}
\usepackage{float,needspace}
\usepackage{url}
\usepackage[hidelinks]{hyperref}

\newcommand{\E}{\mathbb{E}}

\newcommand{\KL}{D_{\mathrm{KL}}}
\newcommand{\Normal}{\mathcal{N}}

\newcolumntype{Y}{>{\raggedright\arraybackslash}X}
\newtheorem{proposition}{Proposition}

\theoremstyle{definition}

\hypersetup{pdftitle={Noise Space Learning for VLA Adaptation},pdfauthor={}}

%% file: sections/abstract.tex
Pretrained vision-language-action (VLA) models encode broad manipulation priors, yet adapting them to specific tasks requires balancing limited robot interaction, computational resources, and safety constraints.
Noise space learning (NSL) adapts flow-matching-based VLAs by learning a context-conditioned distribution over noise while keeping the pretrained VLA fixed.
Building on prior noise-steering methods, we formulate NSL as an interface for VLA adaptation that accommodates different learning algorithms and instantiate it with PPO, SAC, and AWR.
\textcolor{black}{Theoretically, we provide an information-theoretic analysis of how NSL enables efficient adaptation.}
Empirically, we study these NSL algorithms on LIBERO, RoboTwin, and four real-robot manipulation tasks under interaction, computation, and execution constraints. 
Across these settings, we develop a suite of effective NSL variants and derive practical recommendations for choosing among them under different resource and deployment regimes.

%% file: sections/introduction.tex
\section{Introduction}
Vision-language-action (VLA) models increasingly combine pretrained visual-language representations with continuous action generation. Early VLAs such as RT-2~\citep{brohan2023_rt2} predict discretized actions as tokens. Octo~\citep{octomodelteam2024_octo} subsequently combines a generalist robot transformer with a diffusion action head to model continuous, multimodal behavior. More recent VLAs, including $\pi_0$~\citep{black2024_pi0} and $\pi_{0.5}$~\citep{physicalintelligence2025_pi05}, use flow matching to transform standard Gaussian noise into action chunks, with the flow generator conditioned on visual observations and language. This progression has made flow matching an increasingly prominent approach to continuous action generation. The broad manipulation priors of flow-matching-based VLAs provide a starting point for deployment, but precise alignment, contact handling, and dynamic manipulation still require adaptation to the target task.

The central challenge in VLA adaptation is to acquire task-specific behavior efficiently while preserving useful pretrained priors. Full-parameter updates to a large VLA are costly and can disturb these priors when training is limited to narrow task experience. For example, low-data supervised fine-tuning (SFT) can leave a policy fixated on training objects or spatial targets, even when a new instruction specifies a different target~\citep{huang2026_delock}. A parameter-efficient alternative is low-rank adaptation (LoRA), which keeps the VLA backbone fixed and trains low-rank adapters. This approach can better preserve VLM-derived visual-semantic recognition than full-parameter fine-tuning~\citep{zhang2025_vlminheritance}, and retain earlier skills with substantially less forgetting in continual VLA adaptation~\citep{hu2026_continualrl}. These observations motivate efficient adaptation that freezes the VLA backbone and learns a small task-specific module.

Keeping the VLA backbone fixed can help preserve pretrained priors and improve adaptation efficiency, but does not by itself ensure effectiveness. Different ways of adapting a fixed backbone face different limitations. Adapting only the action head can underfit task execution despite better preservation of visual-semantic knowledge~\citep{zhang2025_vlminheritance}. Inference-time candidate selection can improve task success, but seeking better actions through a broader search requires generating and scoring more candidates, increasing inference cost~\citep{yang2025_taco}. Residual adaptation learns action corrections, but keeping these corrections close to the base policy for stable learning can limit the behavioral changes it can achieve~\citep{ankile2025_resfit}. These trade-offs motivate a central question:

{\centering
\textbf{\textit{How can we adapt a VLA effectively and efficiently while preserving its priors?}}\par
}

To address this question, we study \emph{noise space learning} (NSL) for flow-matching-based VLAs. NSL trains a lightweight \emph{noise actor} to produce a context-conditioned distribution over the noise that initializes action generation while keeping the VLA backbone fixed. This formulation builds on learning a conditional noise actor in DSRL~\citep{wagenmaker2025_dsrl} and noise supervision obtained by inverting corrective actions in UniSteer~\citep{lu2026_unisteer} and FlowDAgger~\citep{murray2026_flowdagger}. NSL targets effective adaptation by favoring useful behaviors already represented by the generator, supports efficient learning through lightweight actor updates, and preserves the pretrained noise-to-action mapping. The practical challenge is how to optimize the noise actor to realize these benefits. We develop NSL variants based on three reinforcement learning paradigms: on-policy policy optimization with PPO~\citep{schulman2017_ppo}, off-policy actor--critic learning with SAC~\citep{haarnoja2018_sac}, and advantage-weighted regression (AWR)~\citep{peng2019_awr} that uses supervised learning as subroutines.
We study how these variants are impacted by different implementation techniques such as noise limit, task balancing, and success/failure balancing.
Based on these results, we conclude which algorithms are effective and which implementation choices yield strong performance, aiming to establish best practices for NSL under different resource and deployment constraints.

We evaluate NSL on two simulation benchmarks, LIBERO~\citep{liu2023_libero} and RoboTwin~\citep{chen2025_robotwin}, and four real-robot manipulation tasks, cup stacking, toy fishing, sequential insertion, and bead placement, under the following constraints:
\textbf{The interaction constraint} refers to the sample budget for adaptation, which differs between simulation and real-robot experiments and influences algorithm choice.
\textbf{The computation constraint} refers to the computational effort allowed in adaptation and may limit the performance.
\textbf{The execution constraint} concerns the safety of learned behavior on physical robots. We observe that noise distributions far from standard Gaussian can be tolerated in simulation but are associated with unsafe motions and safety interventions, such as emergency stops, on real robots. 

Our contributions are threefold. 
First, we formulate NSL as a VLA adaptation interface that maximizes success rate through a context-conditioned noise actor that steers action generation while keeping the VLA backbone fixed. 
\textcolor{black}{Second, we analyze NSL's adaptation efficiency through an information-theoretic comparison with generator adaptation and residual adaptation.}
Third, we develop NSL variants based on PPO, SAC, and AWR through implementation refinements, and evaluate them on simulation and real robots. 
Our experiments show that NSL-SAC performs well in simulation with ample interaction budgets and requires more computation. However, its noise distribution can deviate from standard Gaussian, with associated unsafe motions and emergency-stop interventions on real robots. 
In contrast, NSL-AWR is better suited to interaction-limited real-robot adaptation, combining lower computational cost with safer execution.
NSL-PPO also achieves competitive performance in simulation, benefiting from more rollouts while showing limited average gains from additional updates.

%% file: sections/related_work.tex
\section{Related Work}

\paragraph{Adapting pretrained VLAs.}
Supervised and parameter-efficient methods update the effective action generator through either full-model fine-tuning or parameter-efficient updates to a limited set of weights~\citep{kim2024_openvla,kim2025_oft,hu2021_lora,jiang2026_vlagse}. Reinforcement-learning methods instead optimize generative policies using denoising or flow-transition likelihoods, critic-derived targets, or policy gradients~\citep{ren2024_dppo,zhang2025_reinflow,chen2025_pirl,shi2026_flowdpg}. These approaches can expand or reshape the behaviors produced by the VLA, with retention and computational cost depending on the update recipe. In continual VLA adaptation, full-parameter RL can cause substantial forgetting of earlier skills, while LoRA-based updates can improve retention~\citep{hu2026_continualrl}. Residual methods such as ResFiT~\citep{ankile2025_resfit} leave the base policy intact and correct its output, whereas hybrid methods such as EXPO-FT~\citep{dong2026_expoft} combine generator fitting with action editing and selection. Inference-time methods instead alter integration or select among generated candidates~\citep{ruan2026_qpilots,yang2025_taco,giridhar2026_qplanning}.
While prior designs aim to balance effectiveness, prior preservation, adaptation and inference efficiency, we argue that NSL provides a better option for this balance.

\paragraph{Learning through noise.}
DSRL~\citep{wagenmaker2025_dsrl} uses a conditional noise actor for frozen diffusion policies and adopts an actor-critic variant.
UniSteer~\citep{lu2026_unisteer}, FlowDAgger~\citep{murray2026_flowdagger}, and Flow Reversal Steering~\citep{tang2026_frs} use approximate inversion to convert action-level corrections into noise supervision. These methods share a noise-space interface and differ in feedback requirements and critic construction. Golden Ticket~\citep{patil2026_goldenticket} improves generative robot policies by searching for a single, context-independent noise vector. LP-DS~\citep{simsir2026_lpds} learns constrained noise perturbations for continuous control and robotic manipulation. LeaP~\citep{dai2026_leap} jointly learns a noise prior and generator for robotic imitation learning. Together, these studies establish noise-distribution design as a promising direction for improving generative robot policies.
We build NSL upon these precedents and formulate it into a unified interface for adapting flow-matching-based VLAs.

%% file: sections/methods.tex
\section{Methods}
\label{sec:methods}
In this section, we first define noise space learning (NSL) as a VLA adaptation approach that shapes the initial noise distribution while fixing the action generator.
Then, we analyze different approaches for VLA adaptation from an information-theoretic perspective to demonstrate the effectiveness of NSL.
Finally, we present different NSL variants that optimize the noise actor using PPO, SAC, and AWR, respectively.

\subsection{Noise Space Learning}
Let $s$ denote the context, typically including image observation, proprioception state and task instruction. Under deterministic inference, a pretrained VLA maps initial noise $z\in\mathcal Z$ to an action chunk $a=G_\phi(s,z)\in\mathcal A$ through a fixed action generating process $G$  parameterized by $\phi$.

The standard noise prior $\rho$, typically $\Normal(0,I)$, induces a base action policy $\mu_\phi^a$. NSL replaces $\rho$ with a learned context-conditioned noise distribution $\pi_\theta^z(\cdot\mid s)$. Denoting $G_{\phi,s}=G_\phi(s,\cdot)$, we have
\begin{equation}
\mu_\phi^a(\cdot\mid s)=(G_{\phi,s})_\#\rho,\qquad
\pi_{\theta,\phi}^a(\cdot\mid s)=(G_{\phi,s})_\#\pi_\theta^z(\cdot\mid s),
\label{eq:induced}
\end{equation}
where $\#$ is the pushforward operator. NSL trains the noise actor while keeping the action generating process fixed. Appendix~\ref{app:information} extends the formulation to stochastic decoding.

At chunk decision $n$, let $R_n$ denote the accumulated reward within the chunk and $\gamma$ the discount factor between chunks. (See Appendix~\ref{app:chunk_objective} for the detailed definition for chunks.) By sampling $z_n\sim\pi_\theta^z(\cdot\mid s_n)$ and executing $a_n=G_\phi(s_n,z_n)$, we obtain a reward sequence. NSL optimizes the noise actor to maximize
\begin{equation}
J(\theta):=
\E\left[\sum_{n\geq 0} \gamma^n R_n\right]
\label{eq:nsl_objective}
\end{equation}

\begin{figure}[t]
\centering
\includegraphics[width=0.93\linewidth]{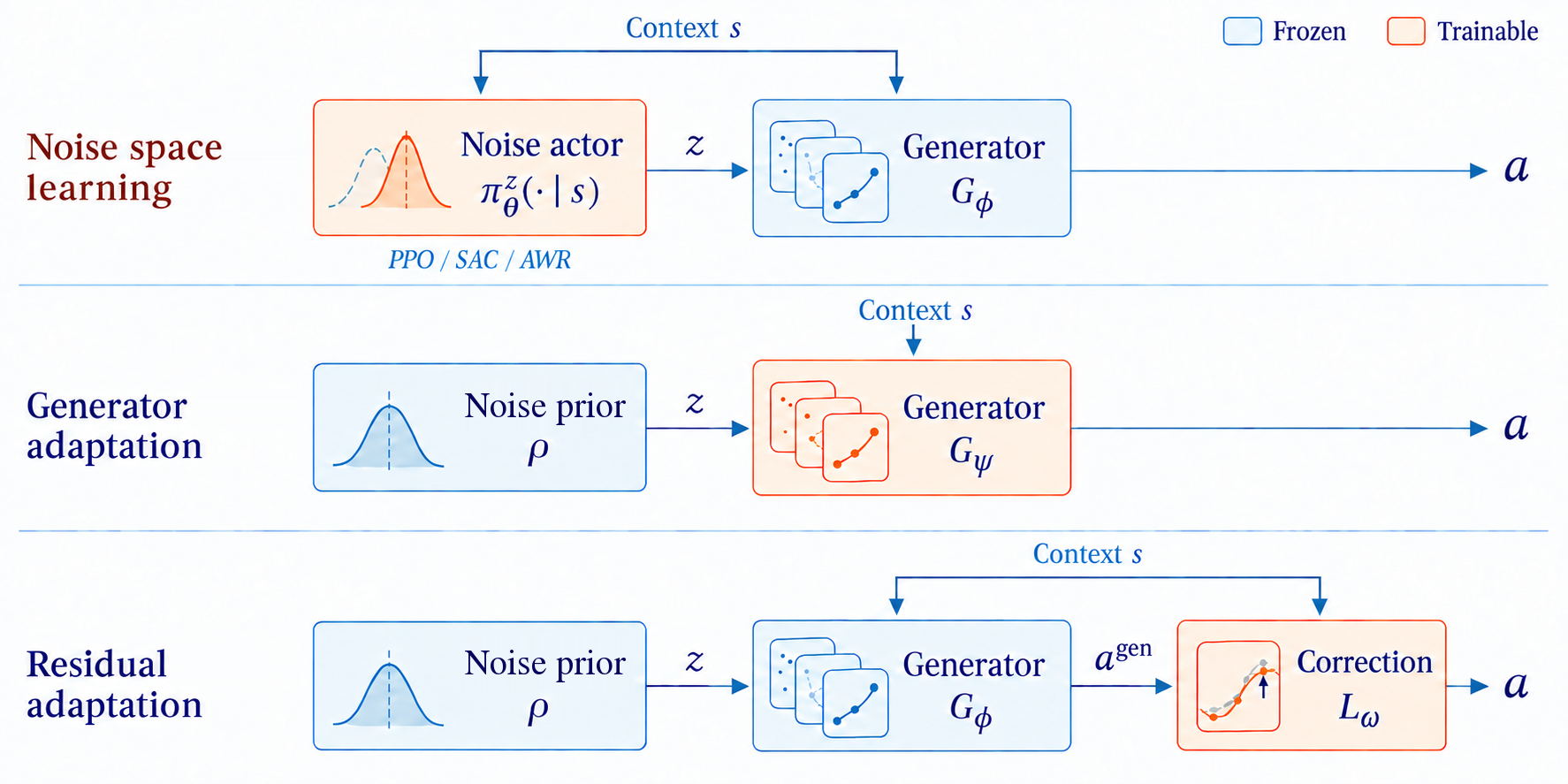}
\caption{Different VLA adaptation approaches for a deterministic generator. NSL learns a context-conditioned noise distribution, generator adaptation updates the noise-to-action map, and residual adaptation modifies generated actions before execution. Orange modules are trainable; blue modules are fixed.}
\label{fig:adaptation_interfaces}
\end{figure}

\subsection{Information-Theoretic View of Different Adaptation Approaches}

From an information-theoretic perspective, given a fixed context, we view the noise prior as the source distribution, and the pretrained generator as a channel that maps noise to actions.
Under this perspective, different approaches adapt different components, illustrated in Figure~\ref{fig:adaptation_interfaces}.
NSL learns a noise actor to adjust the source distribution, generator adaptation changes the channel, and residual adaptation further modifies the output.
Then, we use KL divergence to characterize the required information costs of different adaptation approaches to achieve a target action distribution.
We establish conditions under which NSL achieves efficient adaptation in an information-theoretic sense, realizing a target action distribution at the minimum information cost.

\paragraph{Different adaptation approaches}
For a fixed context $s$, let $A^{\mathrm{gen}}$ denote the generator output and $A$ the action chunk after correction.
Write $K_\phi(\cdot\mid s,z)=\delta_{G_\phi(s,z)}$ for the generator kernel and $L_0(\cdot\mid s,a^{\mathrm{gen}})=\delta_{a^{\mathrm{gen}}}$ for identity correction.
The reference and adapted joint distributions on $(Z,A^{\mathrm{gen}},A)$ are
\begin{equation}
\begin{aligned}
P_0&=\rho K_\phi L_0, & P_{\mathrm{NSL}}&=\pi^z K_\phi L_0,\\
P_{\mathrm{gen}}&=\rho K_\psi L_0, & P_{\mathrm{corr}}&=\rho K_\phi L_\omega.
\end{aligned}
\label{eq:intervention_locations}
\end{equation}
The three adaptations modify noise sampling, generation, and correction, respectively.

\paragraph{A common bound for behavioral change.}
Fix a context distribution $d$ shared across adaptation approaches and independent of the adapted policy.
Let $q_M^a$ denote the marginal distribution of $A$ under $P_M$, where $M\in\{\mathrm{NSL},\mathrm{gen},\mathrm{corr}\}$.
The corresponding marginal under $P_0$ is $\mu_\phi^a$.
The data-processing inequality~\citep{vanerven2012_divergence} gives a behavioral lower bound on the intervention information cost:
\begin{equation}
\begin{aligned}
\mathcal C_M&:=\E_{s\sim d}\KL(P_M(\cdot\mid s)\|P_0(\cdot\mid s))\\
&\geq\mathcal B_M:=\E_{s\sim d}\KL(q_M^a(\cdot\mid s)\|\mu_\phi^a(\cdot\mid s)).
\end{aligned}
\label{eq:information_cost}
\end{equation}
For NSL, $\mathcal C_{\mathrm{NSL}}=\E_{s\sim d}\KL(\pi^z(\cdot\mid s)\|\rho)$; for other approaches, this term becomes more complex (see Appendix~\ref{app:interventions} for details).

\begin{proposition}[NSL achieves minimum information cost]
\label{prop:lift}
Fix a context $s$ and a measurable deterministic generator $G_{\phi,s}:\mathcal Z\to\mathcal A$ on standard Borel spaces.
Let $\mu_\phi^a=(G_{\phi,s})_\#\rho$, and let the target action distribution $q\ll\mu_\phi^a$\footnote{$q \ll \mu_\phi^a$ means that $q$ is absolutely continuous with respect to $\mu_\phi^a$: any event with zero probability under $\mu_\phi^a$ also has zero probability under $q$.} have finite $\KL(q\|\mu_\phi^a)$.
Over all noise distributions on $\mathcal Z$,
\begin{equation}
\min_{\pi^z:(G_{\phi,s})_\#\pi^z=q}\KL(\pi^z\|\rho)=\KL(q\|\mu_\phi^a).
\label{eq:minimum_lift}
\end{equation}
The minimum is attained by a noise distribution $\pi^{z,*}$ satisfying
\[
\frac{d\pi^{z,*}}{d\rho}(z)=\frac{dq}{d\mu_\phi^a}\!\left(G_{\phi,s}(z)\right).
\]
\end{proposition}
See Appendix~\ref{app:minimum_lift} for the proof.

\paragraph{Implications for downstream adaptation.}
For a bounded measurable utility $u(s,a)$, write $\E_{d,q}u:=\E_{s\sim d,\,a\sim q(\cdot\mid s)}[u(s,a)]$.
Using the same $d$, $u$, threshold $\tau$, and reference law $P_0$, define the infimal behavioral and intervention information costs under the average utility requirement:
\begin{equation}
\begin{aligned}
\mathfrak R_{\mathrm{beh}}(\tau)
&=\inf_{q:\E_{d,q}u\geq\tau}\E_{s\sim d}\KL(q(\cdot\mid s)\|\mu_\phi^a(\cdot\mid s)), \\
\mathfrak R_M(\tau)
&=\inf_{\xi_M:\E_{d,q_M^a}u\geq\tau}\mathcal C_M(\xi_M),
\qquad M\in\{\mathrm{NSL},\mathrm{gen},\mathrm{corr}\}.
\end{aligned}
\end{equation}
Here $\xi_M$ ranges over the admissible adapted factors, with $\xi_{\mathrm{NSL}}=\pi^z$, $\xi_{\mathrm{gen}}=K_\psi$, and $\xi_{\mathrm{corr}}=L_\omega$.

\begin{proposition}[Information costs for different adaptation approaches]
\label{prop:utility}
Let $G_\phi:\mathcal S\times\mathcal Z\to\mathcal A$ be a measurable deterministic generator on standard Borel spaces.
Under the common $d,u,\tau,P_0$ defined above, suppose that NSL admits all measurable conditional noise distributions and that $\mathfrak R_{\mathrm{beh}}(\tau)<\infty$. Then
\begin{equation}
\mathfrak R_{\mathrm{NSL}}(\tau)=\mathfrak R_{\mathrm{beh}}(\tau)
\leq\mathfrak R_M(\tau), M\in\{\mathrm{gen},\mathrm{corr}\}.
\label{eq:utility_lift}
\end{equation}
\end{proposition}
NSL therefore matches the behavioral information lower bound while retaining the pretrained generator. See Appendix~\ref{app:utility} for the full definitions and proof.

\subsection{Optimization Methods for Noise Space Learning}
We use conditional diagonal Gaussian noise actors with positive variance and prior-matched initialization. PPO, SAC, and AWR update the noise actor and auxiliary value models without backpropagating through the frozen VLA. Write $\mathcal D$ for replay; Appendix~\ref{app:optimizer_details} gives the transition and target definitions.

\paragraph{NSL-PPO.}
Using recent rollouts and estimated advantages $\widehat A_n$, PPO~\citep{schulman2017_ppo} maximizes
\begin{equation}
\E\left[\min\left(r_n(\theta)\widehat A_n,
\operatorname{clip}(r_n(\theta),1-\epsilon,1+\epsilon)\widehat A_n\right)\right],
\quad r_n(\theta)=\frac{\pi_\theta^z(z_n\mid s_n)}{\pi_{\mathrm{old}}^z(z_n\mid s_n)}.
\label{eq:ppo}
\end{equation}
Here $\epsilon\in(0,1)$ sets the clipping range; the likelihood ratio is evaluated directly under the noise actor.

\paragraph{NSL-SAC.}
SAC~\citep{haarnoja2018_sac} learns twin noise-space critics $Q_{\chi_i}^z$ from replay and minimizes
\begin{equation}
\E_{s\sim\mathcal D,\,z\sim\pi_\theta^z}
\left[\alpha\log\pi_\theta^z(z\mid s)-\min_{i=1,2}Q_{\chi_i}^z(s,z)\right].
\label{eq:sac}
\end{equation}
Here $\alpha\geq0$ is the entropy temperature. Actor gradients pass through the noise-space critic.

\paragraph{NSL-AWR.}
Our success-filtered AWR~\citep{peng2019_awr} variant fits the noise actor to replay noises with estimated advantage weights:
\begin{equation}
\max_\theta\E_{(s,z,\widehat A,Y)\sim\mathcal D}
\left[w\log\pi_\theta^z(z\mid s)\right],\qquad
w=Y\min\{\exp(\widehat A/\beta),w_{\max}\}.
\label{eq:awr}
\end{equation}
Here $Y$ is the trajectory success indicator, $\beta>0$ is the temperature, and $w_{\max}>0$ caps the weights. Appendix~\ref{app:awr} derives the weighted-likelihood projection from KL-regularized improvement~\citep{peng2019_awr,levine2018_inference} and describes the effects of filtering and clipping.

%% file: sections/experiments.tex
\section{Experiments}
\label{sec:experiments}
\subsection{Experimental Setup}
\label{sec:experimental_setup}
\paragraph{Benchmark and models.}
We evaluate SmolVLA~\citep{shukor2025_smolvla} and $\pi_{0.5}$ on LIBERO-10~\citep{liu2023_libero}, also known as LIBERO-Long. For each backbone, we use 10,000-step SFT checkpoints trained on either one demonstration per task (\emph{One}, 10 total) or all 379 filtered demonstrations (\emph{Full}), yielding four base policies. Each NSL run trains one shared, task-conditioned noise actor across all ten tasks. Compared methods start from the same checkpoint within each backbone--data setting. We present LIBERO results in the main text and RoboTwin results in Appendix~\ref{app:robotwin}.

\paragraph{Methods and baselines.}
We compare NSL-PPO, NSL-SAC, and NSL-AWR. All three freeze the VLA and condition the noise actor on pooled VLA hidden features and robot proprioception. Baselines include the base VLA policy, Action-AWR for generator adaptation, and RLT for residual adaptation.

\paragraph{Experimental settings.}
On LIBERO-10, each training epoch covers 50 benchmark reset states per task, with eight stochastic rollouts per task--state pair, yielding 4,000 trajectories. Mixed-task batches contain up to 64 such pairs. Eight workers collect experience for a centralized learner, and task-aware replay balances tasks and samples successful and failed experience. AWR restricts actor regression to successful samples while training its value model on both outcomes. SAC and AWR retain replay across collection batches; PPO clears its buffer after each collection--update cycle. The policies predict 50 actions and execute the first ten before replanning, with a maximum episode length of 520 environment steps.

\paragraph{Evaluation metrics.}
After each epoch, we evaluate all 500 task--reset pairs in LIBERO-10 once and report the mean success rate across all ten tasks. We report the main learning curves over 60 epochs and the budget and component ablations over the first 30 epochs.

\paragraph{Real-robot tasks.}
We study cup stacking, toy fishing, sequential insertion, and bead placement using task-specific $\pi_{0.5}$ policies initialized by supervised fine-tuning. Within each task, methods start from the same base policy. We compare NSL-AWR with Action-AWR and RLT under autonomous adaptation, and NSL-AWR with human guidance against DAgger~\citep{ross2011_dagger} and UniSteer. Each method is evaluated over 40 trials per task; we report success rates and cumulative training time averaged over tasks. Results appear in Section~\ref{sec:real_robot_experiments}, with task definitions and collection details in Appendix~\ref{app:real_robot_protocol}.

\subsection{Results of Noise Space Learning}
\label{sec:main_results}
\begin{figure}[t]
\centering
\includegraphics[width=0.82\linewidth]{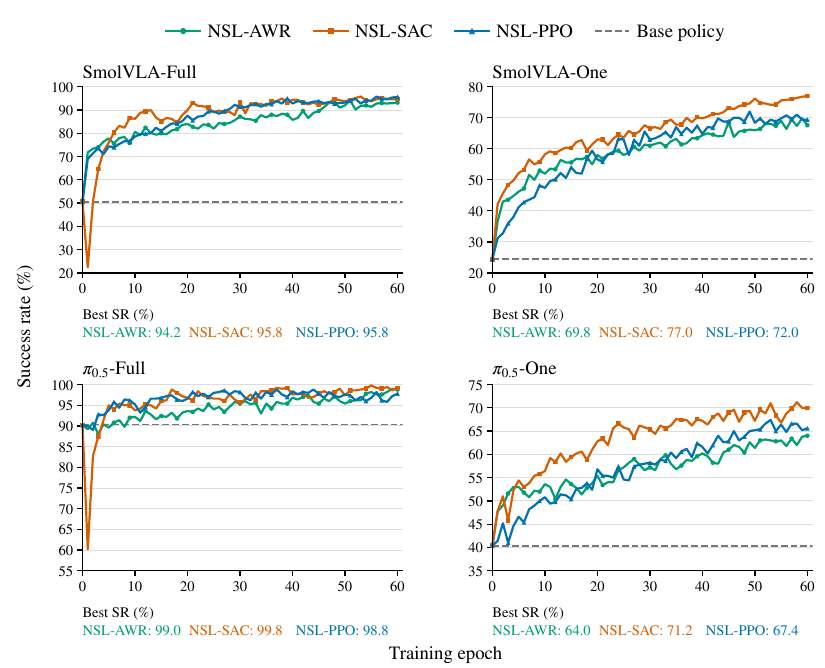}
\caption{Noise space learning on LIBERO-10 with SmolVLA and $\pi_{0.5}$ under Full and One demonstration settings. Curves show success rates over 500 evaluation episodes at each epoch; dashed lines denote the corresponding base-policy references. Values below each panel report the best observed success rate within 60 epochs.}
\label{fig:main_results}
\end{figure}

Figure~\ref{fig:main_results} presents the learning curves of the three NSL methods across four base policies. All three improve over their base-policy references. NSL-SAC achieves the highest peak scores, tying with NSL-PPO on SmolVLA-Full. In the Full setting, peak success rates reach 95.8\% for SmolVLA and 99.8\% for $\pi_{0.5}$; in the One setting, they rise from 24.4\% and 40.4\% to 77.0\% and 71.2\%, respectively. NSL-PPO approaches NSL-SAC in both Full settings and exceeds NSL-AWR's peak scores in both One settings. With the VLA frozen throughout, these gains show that learning only a conditional noise distribution can effectively improve the pretrained generator's task performance. This is consistent with the noise-reweighting mechanism characterized by Proposition~\ref{prop:lift}: redistributing noise probabilities changes the generated behavior distribution while preserving the generator map. The high Full scores and substantial One gains demonstrate the practical effectiveness of this mechanism across demonstration budgets. Appendix~\ref{app:awr_noise_evolution} examines NSL-AWR's noise evolution on fixed observations.

These results use AWR $1\times$, PPO $2\times$, and SAC $4\times$ update budgets. The methods share the collection schedule but differ in optimization compute. We examine this dependence through update-budget ablations, varying the number of updates while holding the collection protocol fixed.

\subsection{Comparison with Generator and Residual Adaptation}
\label{sec:action_expert_comparison}
Table~\ref{tab:training_time} compares NSL with Action-AWR for generator adaptation and RLT for residual adaptation. Action-AWR freezes the VLM and updates the action expert through advantage-weighted flow matching on successful samples. Under the same collection schedule, NSL-SAC and NSL-PPO achieve average SRs of 83.5\% and 81.6\%, respectively, compared with RLT's 81.0\% and Action-AWR's 79.4\%. NSL-AWR achieves a competitive average SR of 79.1\%, close to Action-AWR. Across the individual model settings, NSL-AWR matches RLT on SmolVLA-Full and falls below both baselines on $\pi_{0.5}$-One.

NSL combines this performance with substantially lower update time. Averaged over the two models, NSL-SAC takes 9.93 seconds per collection batch, approximately $1/21.2$ of Action-AWR's update time and $1/11.2$ of RLT's. NSL-AWR delivers its competitive performance with the lowest update time of 1.09 seconds per batch. NSL-PPO also achieves an average success rate comparable to RLT's at a lower update cost. From an information-cost perspective, Proposition~\ref{prop:utility} establishes that NSL can attain the common behavioral information lower bound while retaining the pretrained generator. The experiments further demonstrate the practical value of this adaptation interface: by optimizing a lightweight noise actor without backpropagating through the VLA, NSL achieves competitive task performance at a low update cost.

\begin{table}[t]
\caption{Performance and update time on two model settings. Adapted-policy SR is the best observed within 60 epochs, with 500 evaluation episodes per epoch. Avg.\ is the equal-weight mean of the two model settings; Base policy gives the unadapted reference. Time is the gradient-update time per collection batch averaged over both models, excluding collection and evaluation.}
\label{tab:training_time}
\centering
\small
\renewcommand{\arraystretch}{1.03}
\setlength{\aboverulesep}{1.4pt}
\setlength{\belowrulesep}{2.0pt}
\setlength{\tabcolsep}{2pt}
\begin{tabular*}{\linewidth}{@{\extracolsep{\fill}}lrrrrrr@{}}
\toprule
Model / metric & Base policy & Action-AWR & RLT & NSL-AWR & NSL-PPO & NSL-SAC \\
\midrule
SmolVLA-Full SR (\%) & 50.6 & 93.0 & 94.2 & 94.2 & \textbf{95.8} & \textbf{95.8} \\
$\pi_{0.5}$-One SR (\%) & 40.4 & 65.8 & 67.8 & 64.0 & 67.4 & \textbf{71.2} \\
Avg. SR (\%) & 45.5 & 79.4 & 81.0 & 79.1 & 81.6 & \textbf{83.5} \\
\midrule
Update time (s/batch) & --- & 210.20 & 111.60 & \textbf{1.09} & 2.34 & 9.93 \\
\bottomrule
\end{tabular*}
\end{table}

\subsection{Analysis of Update Budget}
\label{sec:update_budget}
We compare update budgets on SmolVLA-Full and $\pi_{0.5}$-One while keeping the rollout protocol fixed. For each algorithm, we scale gradient updates per collection batch to $1\times$, $2\times$, and $4\times$ its baseline schedule. Figure~\ref{fig:update_budget_comparison} reports the mean success rate over the two models.

\begin{figure}[t]
\centering
\includegraphics[width=\linewidth]{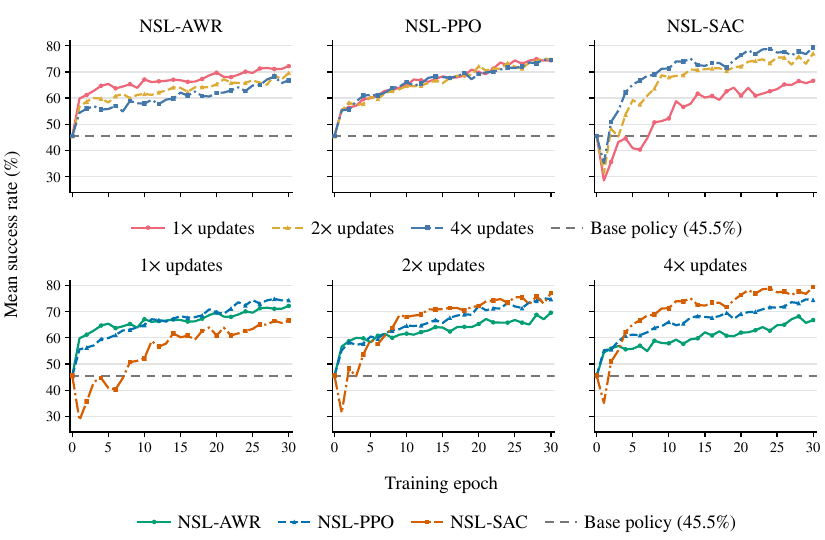}
\caption{Effect of update budget, averaged equally over SmolVLA-Full and $\pi_{0.5}$-One. Top: $1\times$, $2\times$, and $4\times$ update budgets within each algorithm. Bottom: algorithms compared at each update budget. The rollout protocol is fixed. Dashed lines show the mean base-policy success rate (45.5\%).}
\label{fig:update_budget_comparison}
\end{figure}

The effect of additional updates depends on the algorithm. NSL-AWR achieves its best performance at $1\times$, with further updates reducing success rates. NSL-SAC benefits substantially from more updates: the largest improvement occurs from $1\times$ to $2\times$, and increasing to $4\times$ yields a smaller further gain. NSL-PPO's mean learning curves remain similar across $1\times$, $2\times$, and $4\times$ update budgets, with limited average performance gains from additional updates.

With a limited update budget, AWR quickly improves success rates early in training, consistent with learning directly from successful experience through weighted regression. Relatively few gradient updates yield meaningful gains, allowing AWR to achieve competitive performance with a low update cost. With a larger update budget, SAC continues to benefit from additional value and policy updates and reaches higher final success rates.

\subsection{Analysis of Rollout Count}
\label{sec:rollout_count}
We compare one, four, and eight rollouts per task--reset pair on SmolVLA-Full and $\pi_{0.5}$-One, keeping the update ratios per collected transition unchanged at AWR $1\times$, PPO $2\times$, and SAC $4\times$. Figure~\ref{fig:rollout_count_comparison} reports the mean success rate over the two models. More rollouts improve all three methods, with SAC benefiting most, although the additional gains from four to eight rollouts are smaller.

\begin{figure}[t]
\centering
\includegraphics[width=\linewidth]{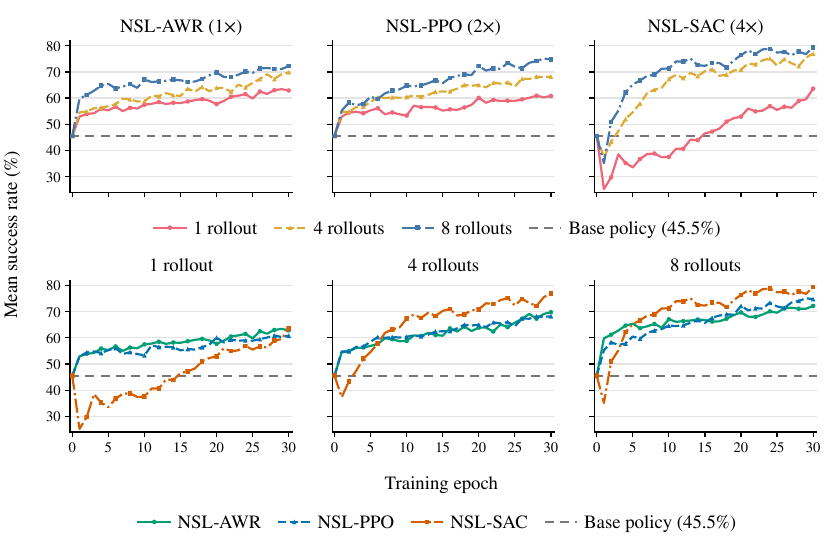}
\caption{Effect of rollout count, averaged equally over SmolVLA-Full and $\pi_{0.5}$-One. Top: one, four, and eight rollouts per task--reset pair within each algorithm. Bottom: algorithms compared at each rollout count. Update ratios remain AWR $1\times$, PPO $2\times$, and SAC $4\times$. Dashed lines show the mean base-policy success rate (45.5\%).}
\label{fig:rollout_count_comparison}
\end{figure}

With limited collection, AWR quickly exploits observed successful experience, achieving strong early performance with relatively smooth learning curves. As collection increases, SAC overtakes AWR earlier and reaches the highest final success rate, consistent with richer state--noise samples supporting critic-guided optimization. PPO also shows more sustained improvement with more rollouts and surpasses AWR later in training with eight rollouts.

Together with the update-budget results, these findings show different resource requirements across the three methods. AWR quickly achieves competitive performance with limited collection and relatively few updates, making it suitable for online adaptation with constrained collection and update budgets. SAC achieves higher final performance when both collection and update budgets are larger. PPO also benefits from more rollouts; under a larger rollout budget, a lower update multiplier already yields mean performance close to that of higher multipliers.

\subsection{Analysis of Training Components}
\label{sec:component_ablations}
We ablate \emph{noise limits} on generator inputs, \emph{task balancing} across rollout batches and replay storage and sampling, and \emph{success/failure balancing} in update batches on SmolVLA-Full and $\pi_{0.5}$-One. Removing all components uses unbounded noise and uniform sampling from a global replay buffer. Table~\ref{tab:component_ablations} summarizes the results.

\begin{table}[t]
\caption{Component ablations. Values are the best success rates (\%) within the first 30 epochs, with changes from the corresponding baseline in parentheses (pp). PPO and AWR already use unbounded noise, so their no-limit entries reuse the baseline.}
\label{tab:component_ablations}
\centering
\small
\renewcommand{\arraystretch}{1.03}
\setlength{\aboverulesep}{1.4pt}
\setlength{\belowrulesep}{2.0pt}
\setlength{\tabcolsep}{3pt}
\begin{tabular*}{\linewidth}{@{\extracolsep{\fill}}llrrrrr@{}}
\toprule
Model & Method & \shortstack{Baseline\\(\%)} & \shortstack{w/o Noise\\Limit} & \shortstack{w/o Task\\Balancing} & \shortstack{w/o Success/Failure\\Balancing} & \shortstack{All\\Removed} \\
\midrule
\multirow{3}{*}{\shortstack[l]{SmolVLA\\Full}} & SAC & 93.4 & $32.0\;(-61.4)$ & $80.6\;(-12.8)$ & $85.2\;(-8.2)$ & $0.0\;(-93.4)$ \\
 & PPO & 92.2 & $92.2\;(0.0)$ & $87.2\;(-5.0)$ & $91.8\;(-0.4)$ & $84.0\;(-8.2)$ \\
 & AWR & 87.2 & $87.2\;(-0.0)$ & $87.2\;(-0.0)$ & $87.0\;(-0.2)$ & $90.2\;(+3.0)$ \\
\midrule
\multirow{3}{*}{$\pi_{0.5}$-One} & SAC & 66.6 & $56.6\;(-10.0)$ & $60.4\;(-6.2)$ & $61.8\;(-4.8)$ & $39.8\;(-26.8)$ \\
 & PPO & 58.2 & $58.2\;(0.0)$ & $50.4\;(-7.8)$ & $58.0\;(-0.2)$ & $52.4\;(-5.8)$ \\
 & AWR & 59.0 & $59.0\;(-0.0)$ & $61.6\;(+2.6)$ & $58.4\;(-0.6)$ & $60.2\;(+1.2)$ \\
\bottomrule
\end{tabular*}
\end{table}

The component effects vary across optimizers. NSL-SAC benefits from task and success/failure balancing on both models and is particularly sensitive to noise limits. NSL-PPO benefits primarily from task balancing: removing it or all components lowers success rates on both models, whereas removing success/failure balancing leads to only small decreases. NSL-AWR maintains comparable performance with all components removed, indicating lower sensitivity to these components.

Figure~\ref{fig:sac_noise_limit} in Appendix~\ref{app:sim_noise_diagnostics} examines NSL-SAC diagnostic runs on SmolVLA-Full. Without a noise limit, the actor's Gaussian means drift farther from initialization on fixed contexts, while the upper tail of critic targets remains above the return bound. Both quantities are lower with bounded noise. This pattern is consistent with value-driven outward drift in noise space. NSL-AWR instead fits observed successful noises through weighted regression, consistent with its lower sensitivity to these components.

\subsection{Real-Robot Experiments}
\label{sec:real_robot_experiments}
NSL-AWR learns from observed successful noises through weighted regression. Sections~\ref{sec:update_budget}--\ref{sec:component_ablations} show that it learns quickly with fewer rollouts and low update budgets and is less sensitive to training components, making it suitable for costly robot interaction with small collection batches. An earlier real-robot bead-placement experiment with NSL-SAC produced unintended table contact. Appendix~\ref{app:safety} reports noise approaching the amplitude limit and actor proposals with higher predicted Q values and greater critic disagreement than replay noise. Together, these factors motivate our choice of NSL-AWR for the real-robot comparison and our study of its use with human corrections.

\begin{figure}[t]
\centering
\includegraphics[width=\linewidth]{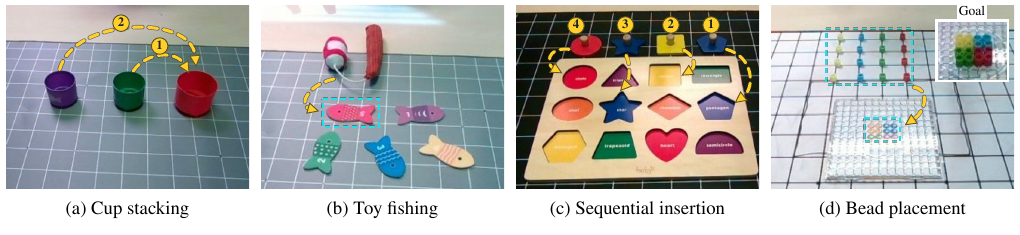}
\caption{Real-robot tasks. Arrows indicate operations, and numbers denote their order. The inset shows a completed bead-placement goal from a successful demonstration.}
\label{fig:real_robot_tasks}
\end{figure}

\begin{table}[htbp]
\caption{Real-robot success rates over 40 trials per task and training time. Avg. SR and cumulative model-update time are averaged over the four tasks.}
\label{tab:real_robot_results}
\centering
\small
\renewcommand{\arraystretch}{1.03}
\setlength{\aboverulesep}{1.4pt}
\setlength{\belowrulesep}{2.0pt}
\setlength{\tabcolsep}{2pt}
\begin{tabular*}{\linewidth}{@{\extracolsep{\fill}}lrrrrrrr@{}}
\toprule
\multirow{2}{*}{Task (SR, \%)} & \multirow{2}{*}{\shortstack{Initial\\policy}} & \multicolumn{3}{c}{\emph{Autonomous}} & \multicolumn{3}{c}{\emph{Human-guided}} \\
\cmidrule(lr){3-5}\cmidrule(lr){6-8}
 & & Action-AWR & RLT & \textbf{NSL-AWR} & DAgger & UniSteer & \shortstack{\textbf{NSL-AWR}\\\textbf{+ human}} \\
\midrule
Cup stacking & 75.0 & 100.0 & 100.0 & 100.0 & 100.0 & 100.0 & 100.0 \\
Toy fishing & 55.0 & 70.0 & 75.0 & 80.0 & 70.0 & 90.0 & 90.0 \\
Sequential insertion & 15.0 & 50.0 & 75.0 & 70.0 & 70.0 & 85.0 & 90.0 \\
Bead placement & 5.0 & 15.0 & 30.0 & 35.0 & 25.0 & 20.0 & 55.0 \\
Avg. & 37.5 & 58.8 & 70.0 & 71.3 & 66.3 & 73.8 & 83.8 \\
\midrule
Training time (min) & --- & 52.4 & 23.2 & 1.53 & 61.6 & 49.8 & 2.04 \\
\bottomrule
\end{tabular*}
\end{table}

Figure~\ref{fig:real_robot_tasks} illustrates cup stacking, toy fishing, sequential insertion, and bead placement. Table~\ref{tab:real_robot_results} shows that autonomous NSL-AWR improves mean success from the initial policy's 37.5\% to 71.3\%, exceeding Action-AWR's 58.8\% and performing comparably to RLT's 70.0\%. Performance varies across individual tasks, with RLT performing better on sequential insertion. Overall, learning a conditional noise distribution improves real-robot behavior with the VLA frozen and achieves performance comparable to generator and residual adaptation.

\begingroup
\looseness=-1
\textls[-20]{With human guidance, NSL-AWR's mean success further increases to 83.8\%, exceeding 66.3\% for DAgger and 73.8\% for UniSteer. The largest gains occur on sequential insertion and bead placement, where success rises from 70\% and 35\% to 90\% and 55\%, respectively. These improvements illustrate how human corrections complement autonomous exploration by supplying useful behaviors that are difficult to obtain through autonomous collection alone, enabling noise space learning to further improve complex manipulation.}\par
\endgroup

NSL-AWR also has a lower training cost. Mean cumulative training time per task is 1.53 minutes for autonomous NSL-AWR, compared with 52.4 minutes for Action-AWR and 23.2 minutes for RLT. With human guidance, NSL-AWR takes 2.04 minutes. Figure~\ref{fig:real_robot_collection} in Appendix~\ref{app:real_robot_protocol} further shows that NSL-AWR with human guidance uses fewer human trajectories per round than UniSteer on every task while achieving higher mean success. These results demonstrate NSL-AWR's practical value on real robots: improving policies at low update cost and further enhancing task performance through human corrections.

%% file: sections/discussion.tex
\section{Discussion and Conclusion}
We study Noise Space Learning, which improves task performance by learning a conditional noise distribution while keeping the VLA frozen. Our theoretical analysis characterizes the relationship between noise reweighting and changes in behavior distributions, and establishes conditions under which NSL attains the common behavioral information lower bound. Simulation and real-robot experiments further show that this adaptation interface effectively uses pretrained action generators, achieving performance competitive with generator and residual adaptation at a low gradient-update cost.

Different optimizers suit different resource conditions. NSL-AWR quickly uses observed successful experience, performs well with limited collection and update budgets, and is less sensitive to training components. NSL-SAC benefits more from additional data and updates, achieving higher success rates in simulation when larger budgets are available. NSL-PPO also achieves competitive performance, benefiting from more rollouts while showing limited average gains from additional updates. These findings suggest choosing the optimizer according to the available interaction and computation budgets. Real-robot experiments demonstrate the practicality of NSL-AWR for online learning with small collection batches, while human corrections further improve performance on complex manipulation tasks.

The adaptation capacity of NSL remains constrained by the fixed generator and the expressiveness of the noise actor. Richer noise parameterizations and combinations of noise learning with selective generator updates are promising directions for further study. Future work should also examine generalization to new tasks and robot embodiments, retention of existing skills, and execution reliability in complex contact interactions.

%% file: appendices/information.tex
\section{Information-Cost Analysis}
\label{app:information}
\subsection{Definitions and Scope}
All logarithms are natural, so KL is measured in nats. For probability measures $P,Q$, define $\KL(P\|Q)=\int\log(dP/dQ)\,dP$ if $P\ll Q$, and $+\infty$ otherwise. Standard Borel spaces ensure regular conditional distributions exist. We condition on a fixed context $s$ unless an expectation over the externally fixed $d$ is displayed. Assume the divergences used in equalities are finite, avoiding subtraction of infinite terms.

For stochastic inference, replace the deterministic map in Eq.~\eqref{eq:induced} with a fixed conditional kernel $K_\phi(da\mid s,z)$. The base and adapted action policies are
\begin{equation}
\mu_\phi^a(da\mid s)=\int K_\phi(da\mid s,z)\rho(dz),\qquad
\pi_{\theta,\phi}^a(da\mid s)=\int K_\phi(da\mid s,z)\pi_\theta^z(dz\mid s).
\label{eq:induced_kernel}
\end{equation}

For NSL, $A=A^{\mathrm{gen}}$ because the correction kernel is the identity. Marginalizing Eq.~\eqref{eq:intervention_locations} over $A^{\mathrm{gen}}$ gives the following laws on $(Z,A)$, with $P_0$ also denoting the reference marginal and $P_\pi$ denoting the marginal of $P_{\mathrm{NSL}}$:
\begin{equation}
P_0(dz,da\mid s)=\rho(dz)K_\phi(da\mid s,z),\qquad
P_\pi(dz,da\mid s)=\pi^z(dz\mid s)K_\phi(da\mid s,z).
\label{eq:joint_source}
\end{equation}
Their action marginals are $\mu_\phi^a$ and $\pi^a=q_{\mathrm{NSL}}^a$. This analysis concerns the law supplied to the \emph{fixed} kernel. If a stochastic decoder uses additional independent randomness, that randomness remains part of the kernel. Allowing an actor to control all decoder randomness defines a different adaptation interface.

\subsection{Noise and Behavioral Information}
Apply the KL chain rule in two orders to Eq.~\eqref{eq:joint_source}. Conditioning first on $Z$, the equal kernels cancel:
\begin{align}
\KL(P_\pi\|P_0)
&=\KL(\pi^z\|\rho)+\E_{Z\sim\pi^z}\KL(K_\phi(\cdot\mid s,Z)\|K_\phi(\cdot\mid s,Z))\\
&=\KL(\pi^z\|\rho).
\end{align}
Conditioning first on $A$ gives
\begin{equation}
\KL(\pi^z\|\rho)=\KL(\pi^a\|\mu_\phi^a)
+\E_{A\sim\pi^a}\KL(P_\pi(Z\mid s,A)\|P_0(Z\mid s,A)).
\label{eq:source_chain}
\end{equation}
Integrating with respect to $d$ gives
\begin{equation}
\mathcal C_{\mathrm{NSL}}=\mathcal B_{\mathrm{NSL}}+
\E_{s\sim d,\,a\sim\pi^a}\KL\bigl(P_\pi(Z\mid s,a)\|P_0(Z\mid s,a)\bigr).
\label{eq:chain_rule}
\end{equation}
Nonnegativity gives the data-processing inequality, with equality exactly when the adapted and reference noise posteriors agree for $d(ds)\pi^a(da\mid s)$-almost every $(s,a)$. No differentiability or invertibility of the generator is needed for this decomposition.

\subsection{Comparing Intervention Locations}
\label{app:interventions}
Let $S\sim d$, with $d$ the fixed context distribution used in Section~\ref{sec:methods}.
For the joint laws on $(Z,A^{\mathrm{gen}},A)$ in Eq.~\eqref{eq:intervention_locations}, absolutely continuous factors and finite costs give
\begin{align}
\mathcal C_{\mathrm{NSL}}&=\E_d\KL(\pi^z(\cdot\mid S)\|\rho),\\
\mathcal C_{\mathrm{gen}}&=\E_{S\sim d,Z\sim\rho}\KL(K_\psi(\cdot\mid S,Z)\|K_\phi(\cdot\mid S,Z)),\\
\mathcal C_{\mathrm{corr}}&=\E_{S\sim d,A^{\mathrm{gen}}\sim\mu_\phi^a}\KL(L_\omega(\cdot\mid S,A^{\mathrm{gen}})\|L_0(\cdot\mid S,A^{\mathrm{gen}})).
\end{align}
When $\mathcal C_M<\infty$, the KL chain rule~\citep{vanerven2012_divergence} with $U=(Z,A^{\mathrm{gen}})$ gives, for $M\in\{\mathrm{NSL},\mathrm{gen},\mathrm{corr}\}$,
\begin{equation}
\begin{aligned}
\mathcal C_M&=\mathcal B_M+\mathcal M_M\ \geq\ \mathcal B_M,\\
\mathcal M_M&=\E_{s\sim d,\,a\sim q_M^a}\KL\bigl(P_M(U\mid s,a)\|P_0(U\mid s,a)\bigr).
\end{aligned}
\label{eq:intervention_bound}
\end{equation}
The conditional mechanism divergence $\mathcal M_M$ measures redistribution of the internal variables given the executed action. This identifies a common behavioral lower bound, not a strict ordering of intervention costs. Hybrid interventions require the full chain rule with expectations under their adapted upstream factors. In particular, EXPO-FT's generator fitting, action edits, and candidate selection cannot be substituted for the single-factor correction law without redefining the internal variables.

\paragraph{Deterministic-kernel singularity.}
Distinct point masses have infinite KL. If deterministic $G_\psi$ and $G_\phi$ differ on a set of positive probability under $d(ds)\rho(dz)$, then $\mathcal C_{\mathrm{gen}}=+\infty$. With $L_0=\delta_{a^{\mathrm{gen}}}$, a correction assigning positive joint probability to $A\ne A^{\mathrm{gen}}$ likewise gives $\mathcal C_{\mathrm{corr}}=+\infty$. For finite-cost comparisons, a common Gaussian relaxation with $\sigma>0$, such as $K_\phi^\sigma=\Normal(G_\phi,\sigma^2I)$ and $K_\psi^\sigma=\Normal(G_\psi,\sigma^2I)$, gives conditional KL $\|G_\psi-G_\phi\|^2/(2\sigma^2)$. The expected cost is finite when the squared displacement is integrable. An analogous shared scale applies to deterministic correction kernels. All marginals and decomposition terms are computed for the relaxed system, and the resulting costs depend on the chosen scale. Proposition~\ref{prop:lift} applies directly to the original NSL system because its deterministic map is unchanged.

\subsection{Proof of the Minimum-Information Lift}
\label{app:minimum_lift}
\begin{proof}[Proof of Proposition~\ref{prop:lift}]
Write $G=G_{\phi,s}$, $\mu_\phi^a=G_\#\rho$, and $f=dq/d\mu_\phi^a$. Define $\pi^*$ by $d\pi^*/d\rho(z)=f(G(z))$. It is a probability measure since
\begin{equation}
\int f(G(z))\rho(dz)=\int f(a)\mu_\phi^a(da)=1.
\end{equation}
For every measurable set $E\subseteq\mathcal A$,
\begin{equation}
\pi^*(G^{-1}(E))=\int_E f(a)\mu_\phi^a(da)=q(E),
\end{equation}
so $G_\#\pi^*=q$. Its information cost is
\begin{equation}
\KL(\pi^*\|\rho)=\int f(G(z))\log f(G(z))\rho(dz)
=\int f(a)\log f(a)\mu_\phi^a(da)=\KL(q\|\mu_\phi^a),
\end{equation}
using $0\log0=0$. Every other noise distribution with pushforward $q$ has cost at least $\KL(q\|\mu_\phi^a)$ by Eq.~\eqref{eq:source_chain}; singular candidates have infinite cost and cannot improve the finite minimum. Thus $\pi^*$ attains the infimum. Its weighting is constant on each fiber $G^{-1}(a)$, preserving the reference noise posterior conditioned on $a$.
\end{proof}

\paragraph{Why deterministic generation matters.}
The data-processing inequality holds for stochastic kernels, but the constructive equality need not. For example, let $K(A=0\mid z)=1/4$ for every $z$. Every noise distribution induces the same action marginal $(1/4,3/4)$. The target $(1/2,1/2)$ is absolutely continuous with respect to this marginal but is unreachable by noise reweighting. For a deterministic generator, absolute continuity with respect to $G_\#\rho$ suffices for the unrestricted construction; this statement is not transferable to an arbitrary stochastic channel.

\paragraph{Fixed noise and coverage.}
A constant-noise policy, as considered by Golden Ticket~\citep{patil2026_goldenticket}, is a useful special case of the control interface but is singular relative to a continuous Gaussian prior. Its KL divergence from that prior is infinite; it lies outside the finite-information proposition. Moreover, a Gaussian prior can have full mathematical support while allocating negligible probability to useful regions. Reachability, high-probability coverage, and learnability under a finite budget are different properties.

\subsection{Utility-Based Information Costs}
\label{app:utility}
For a bounded measurable fixed-context utility $u(s,a)$ and threshold $\tau$, define
\begin{equation}
\mathfrak R_{\mathrm{beh}}(\tau)=
\inf_{q:\E_{d,q}u\geq\tau}\E_d\KL(q(\cdot\mid S)\|\mu_\phi^a(\cdot\mid S)).
\end{equation}
For each interface, use the same $d$, $u$, $\tau$, and reference law $P_0$. Let $\xi_{\mathrm{NSL}}=\pi^z$, $\xi_{\mathrm{gen}}=K_\psi$, and $\xi_{\mathrm{corr}}=L_\omega$ range over their admissible adapted factors. Define
\begin{equation}
\mathfrak R_M(\tau)
=\inf_{\xi_M:\E_{d,q_M^a}u\geq\tau}\mathcal C_M(\xi_M),
\qquad M\in\{\mathrm{NSL},\mathrm{gen},\mathrm{corr}\},
\label{eq:utility_interfaces}
\end{equation}
with value $+\infty$ for an empty feasible set.

\begin{proof}[Proof of Proposition~\ref{prop:utility}]
Every feasible intervention induces a feasible action policy. Data processing gives
\begin{equation}
\mathcal C_M(\xi_M)\geq\E_d\KL(q_M^a\|\mu_\phi^a)\geq\mathfrak R_{\mathrm{beh}}(\tau).
\end{equation}
Taking infima over the admissible factors proves $\mathfrak R_M(\tau)\geq\mathfrak R_{\mathrm{beh}}(\tau)$ for all three interfaces, including empty feasible sets and infinite intervention costs.

To prove the reverse inequality for NSL, fix $\varepsilon>0$. Since $\mathfrak R_{\mathrm{beh}}(\tau)<\infty$, there is a feasible measurable action kernel $q_\varepsilon$ with
\begin{equation}
\E_d\KL(q_\varepsilon(\cdot\mid S)\|\mu_\phi^a(\cdot\mid S))
\leq\mathfrak R_{\mathrm{beh}}(\tau)+\varepsilon.
\end{equation}
Write $Q_\varepsilon(ds,da)=d(ds)q_\varepsilon(da\mid s)$ and $Q_0(ds,da)=d(ds)\mu_\phi^a(da\mid s)$. Finite integrated KL implies $Q_\varepsilon\ll Q_0$. Choose a nonnegative jointly measurable Radon--Nikodym derivative $f_\varepsilon=dQ_\varepsilon/dQ_0$. The function $h_\varepsilon(s)=\int f_\varepsilon(s,a)\mu_\phi^a(da\mid s)$ is measurable and equals $1$ for $d$-almost every $s$, since both joint laws have context marginal $d$. Define
\begin{equation}
\pi_\varepsilon^z(dz\mid s)=
\begin{cases}
f_\varepsilon(s,G_\phi(s,z))\rho(dz),&h_\varepsilon(s)=1,\\
\rho(dz),&h_\varepsilon(s)\ne1.
\end{cases}
\end{equation}
Measurability of $G_\phi$ makes this a measurable probability kernel, normalized at every context. Its action pushforward, integrated against $d$, equals $Q_\varepsilon$. Uniqueness of conditional distributions on standard Borel spaces therefore gives $(G_{\phi,s})_\#\pi_\varepsilon^z(\cdot\mid s)=q_\varepsilon(\cdot\mid s)$ for $d$-almost every $s$. This conditional noise distribution is admissible by the unrestricted-family assumption and has the same utility as $q_\varepsilon$.

The density construction is the minimum-information lift of Proposition~\ref{prop:lift} at almost every context, so
\begin{equation}
\mathcal C_{\mathrm{NSL}}(\pi_\varepsilon^z)
=\E_d\KL(q_\varepsilon(\cdot\mid S)\|\mu_\phi^a(\cdot\mid S))
\leq\mathfrak R_{\mathrm{beh}}(\tau)+\varepsilon.
\end{equation}
Taking the infimum over noise distributions and then letting $\varepsilon\downarrow0$ yields $\mathfrak R_{\mathrm{NSL}}(\tau)\leq\mathfrak R_{\mathrm{beh}}(\tau)$. Together with the shared lower bound, this proves Eq.~\eqref{eq:utility_lift}. The argument does not require an attained utility optimizer.
\end{proof}

The comparison establishes that NSL with unrestricted noise distributions matches the common behavioral requirement in the deterministic system. Strict separation from another interface would require a positive gap in its own infimum. For the stochastic relaxations in Appendix~\ref{app:interventions}, the shared lower bound uses the relaxed reference marginal and induced policies throughout; equality for the unchanged deterministic generator is a separate result.

\subsection{Regularized Noise Actor Improvement and AWR}
\label{app:awr}
Fix $s$, a reference noise distribution $\eta$, a measurable value score $Q(z)$, and $\beta>0$. Assume $0<\Lambda=\int e^{Q(z)/\beta}\eta(dz)<\infty$ and the required expectations exist. Define $q^*(dz)=e^{Q(z)/\beta}\eta(dz)/\Lambda$. For any admissible $q\ll\eta$,
\begin{equation}
\E_q Q-\beta\KL(q\|\eta)=\beta\log\Lambda-\beta\KL(q\|q^*).
\label{eq:gibbs_identity}
\end{equation}
Nonnegativity proves optimality of $q^*$ for the regularized objective. Subtracting a context-only baseline from $Q$ leaves the normalized conditional target unchanged. This is the standard regularized-improvement argument underlying weighted regression~\citep{peng2019_awr,levine2018_inference}, not a new NSL-specific theorem.

For replay law $P_{\mathcal D}(ds,dz)=d_{\mathcal D}(ds)\eta_{\mathcal D}(dz\mid s)$, take any nonnegative weight $w$ with $0<\overline\Lambda=\E_{P_{\mathcal D}}w<\infty$. Set
\begin{equation}
\Lambda(s)=\E_{\eta_{\mathcal D}}w(s,Z),\quad
q^*(dz\mid s)=\frac{w(s,z)\eta_{\mathcal D}(dz\mid s)}{\Lambda(s)},\quad
\widetilde d(ds)=\frac{\Lambda(s)d_{\mathcal D}(ds)}{\overline\Lambda}.
\end{equation}
The conditional target is needed only where $\Lambda(s)>0$. On the weighted context distribution,
\begin{equation}
\E_{\widetilde d}\KL(q^*(\cdot\mid S)\|\pi_\theta^z(\cdot\mid S))
=\mathrm{const}-\frac{1}{\overline\Lambda}\E_{P_{\mathcal D}}[w(S,Z)\log\pi_\theta^z(Z\mid S)].
\label{eq:awr_projection}
\end{equation}
Thus weighted likelihood is an exact forward-KL projection of the population target under the \emph{reweighted} context marginal, not generally under the original $d_{\mathcal D}$. The identity assumes finite cross-entropies. For a discrete empirical replay measure fitted by a continuous density, their measure-level KL is infinite; the finite empirical weighted-log-density objective is instead a Monte Carlo approximation to this population projection, not a literal finite empirical KL.

\paragraph{Filtering and clipping.}
With random success $Y$ and advantage estimate $\widehat A$, the effective population weight is
\begin{equation}
\overline w(s,z)=\E\left[Y\min\{e^{\widehat A/\beta},w_{\max}\}\mid s,z\right].
\end{equation}
Replacing $w$ by $\overline w$ gives the filtered target. Filtering and clipping generally change the unclipped Gibbs optimum. A statewise advantage baseline cancels in an \emph{unclipped conditional} target, but can change context weights in a shared-parameter empirical fit; clipping can also change conditional relative weights. With no successful replay, the weighted objective supplies no actor-learning signal. These qualifications prevent conflating success-only BC, standard AWR, and the authors' filtered variant.

\subsection{Why Behavior Anchoring Is Not a Safety Bound}
For $\pi^z=\Normal(m,\operatorname{diag}(\sigma^2))$ and $\rho=\Normal(0,I)$ in dimension $d_z$,
\begin{equation}
\KL(\pi^z\|\rho)=\tfrac12\sum_{j=1}^{d_z}(m_j^2+\sigma_j^2-1-\log\sigma_j^2),\qquad
\mathcal H(\pi^z)=\tfrac12\sum_{j=1}^{d_z}\log(2\pi e\sigma_j^2).
\end{equation}
Entropy contains no mean penalty and therefore does not control mean drift relative to the pretrained prior. Conversely, weighted maximum likelihood can fit a replay distribution whose mean is already large; a Gaussian fit can also assign mass to regions absent from finite replay. A moving policy reference does not provide a fixed-prior bound either: for $\pi_k=\Normal(kv,I)$, $\KL(\pi_{k+1}\|\pi_k)=\|v\|^2/2$ but $\KL(\pi_K\|\rho)=K^2\|v\|^2/2$. This is an illustration about local versus global constraints, not a literal guarantee supplied by PPO clipping.

No inequality above translates noise moments into collision, force, or joint-limit guarantees. Such a guarantee would require additional assumptions on the generator, environment, controller, and safe set. Physical safety remains an experimental and engineering requirement.

%% file: appendices/protocol.tex
\section{Experimental Details and Additional Results}
\label{app:protocol}
\subsection{Implementation Details}
\label{app:optimizer_details}
\paragraph{Network architecture and noise transformations.}
For LIBERO, all three noise actors take pooled VLA hidden features concatenated with proprioception as input. Each actor has four hidden layers of width 512, with LayerNorm and ReLU after each linear layer, and separate 32-dimensional mean and log-standard-deviation heads. The heads are initialized to zero, and log standard deviations are clamped to $[-20,2]$. The sampled noise vector is repeated across the 50-step prediction horizon. In the main SAC setting, the noise is bounded to $(-1,1)$ after the initial Gaussian exploration batch.

PPO and AWR each use a separate value MLP with four 512-unit ReLU hidden layers and a 16-bin categorical output over $[0,1]$, trained with cross-entropy loss. SAC uses two independent scalar-output critics, each with four 512-unit hidden layers, LayerNorm, and ReLU; their inputs concatenate the context features and the noise vector.

\paragraph{Hyperparameters.}
Table~\ref{tab:nsl_hyperparameters} lists the LIBERO settings. Actors and value/critic networks use AdamW with $(\beta_1,\beta_2)=(0.9,0.95)$, $\epsilon_{\mathrm{Adam}}=10^{-8}$, weight decay $0.01$, and gradient-norm clipping at $1.0$.
AWR and SAC use log probabilities averaged over noise dimensions. For the PPO ratio in Eq.~\eqref{eq:ppo}, joint log probabilities are obtained by summing over noise dimensions. PPO entropy and reference-KL terms use per-dimension averages.

\begin{table}[htbp]
\caption{Hyperparameters for the main LIBERO experiments. SAC learns its entropy temperature, initialized at $0.006$, with Adam at learning rate $3\times10^{-4}$. Its target entropy is specified per noise dimension.}
\label{tab:nsl_hyperparameters}
\centering
\small
\begin{tabular}{lccc}
\toprule
Parameter & NSL-AWR & NSL-PPO & NSL-SAC \\
\midrule
Actor learning rate & $10^{-5}$ & $10^{-5}$ & $10^{-5}$ \\
Value/critic learning rate & $2\times10^{-5}$ & $2\times10^{-5}$ & $2\times10^{-5}$ \\
Training minibatch size & 2,048 & 2,048 & 2,048 \\
Discount per chunk & 0.99 & 0.99 & 0.99 \\
Replay success sampling fraction & 0.5 & 0.5 & 0.5 \\
Random seed & 42 & 42 & 42 \\
\midrule
AWR temperature $\beta$ & 0.5 & --- & --- \\
AWR maximum weight $w_{\max}$ & 20 & --- & --- \\
GAE $\lambda$ & --- & 0.95 & --- \\
PPO clipping range $\epsilon$ & --- & 0.2 & --- \\
Actor entropy coefficient & 0 & 0.01 & Learned \\
Reference-KL coefficient & --- & 0.005 & --- \\
SAC target entropy & --- & --- & 0.5 \\
SAC target-update coefficient $\tau$ & --- & --- & 0.005 \\
SAC Monte Carlo auxiliary weight & --- & --- & 0.5 \\
\bottomrule
\end{tabular}
\end{table}

\paragraph{Definition of the update multiplier.}
Let $N$ be the number of new chunk transitions in a collection batch, $B=2{,}048$ the optimization minibatch size, and $u\in\{1,2,4\}$ the update multiplier. AWR and PPO schedule $10uN/B$ joint actor--value updates and $10uN/B$ additional value-only updates, rounded to the nearest integer. SAC schedules $20uN/B$ twin-critic updates, with one actor update every two critic updates; fractional critic-update budgets are carried forward. Thus, $1\times$ corresponds to approximately $10N/B$ actor updates and $20N/B$ value/critic updates. For $N=20{,}480$, this gives 100 actor updates and 200 value/critic updates. Table~\ref{tab:nsl_update_counts} gives the corresponding counts for the main configurations.

\begin{table}[htbp]
\caption{Scheduled updates for a collection batch containing 20,480 new chunk transitions. Counts scale with the actual number of collected transitions. Each SAC critic update jointly optimizes both Q networks.}
\label{tab:nsl_update_counts}
\centering
\begin{tabular}{lcrr}
\toprule
Method & Main update multiplier & Actor updates & Value/critic updates \\
\midrule
NSL-AWR & $1\times$ & 100 & 200 \\
NSL-PPO & $2\times$ & 200 & 400 \\
NSL-SAC & $4\times$ & 400 & 800 \\
\bottomrule
\end{tabular}
\end{table}

\paragraph{Chunk-level rewards and discounts.}
\label{app:chunk_objective}
Index policy decisions by $n=0,1,\ldots$, and let $c_n$ be the number of low-level controls actually executed before the next decision. Decision $n$ starts at environment step $t_n=\sum_{j<n}c_j$, with $t_0=0$. With per-control rewards $r_t$, the LIBERO experiments sum rewards within each chunk and apply a fixed discount between chunks:
\begin{equation}
R_n=\sum_{j=0}^{c_n-1}r_{t_n+j},\qquad
\gamma=0.99.
\label{eq:chunk_rewards}
\end{equation}
Targets distinguish environment termination from time-limit truncation. Each recorded transition contains the noise that generated the chunk, context, reward, successor context, executed duration, and termination and truncation indicators.

\paragraph{Optimizer specification.}
Section~\ref{sec:methods} gives the actor objectives. PPO advantages use the chunk rewards $R_n$ and the discount factor $\gamma$ between chunks. For a nonterminal chunk, a SAC soft backup is
\begin{equation}
R_n+\gamma\E_{z'\sim\pi_\theta^z(\cdot\mid s')}
\left[\min_{i=1,2}Q_{\bar\chi_i}^z(s',z')-\alpha\log\pi_\theta^z(z'\mid s')\right],
\label{eq:sac_backup}
\end{equation}
where $Q_{\bar\chi_i}^z$ are target critics. In the LIBERO configuration, the critic target omits the entropy term; the critic loss also includes Monte Carlo return regression with the weight in Table~\ref{tab:nsl_hyperparameters}.

If replay contains no successful samples, the success-filtered actor update is skipped.

\paragraph{Evaluation modes.}
In the LIBERO-10 comparison, NSL methods use deterministic actor-mean evaluation, while the base-policy reference uses native Gaussian noise sampling. Stochastic training-rollout success is recorded separately. Learning curves are indexed by training epoch.

\subsection{Real-Robot Protocol}
\label{app:real_robot_protocol}
\paragraph{Policy and execution.}
Each task uses a separately fine-tuned $\pi_{0.5}$ base policy, shared by the compared methods. Observations comprise primary-view and wrist-camera images, the task instruction, and a seven-dimensional proprioceptive vector containing end-effector position, Euler-angle orientation, and gripper state. Primary-view crops have resolution $224\times224$. The action representation likewise contains six pose coordinates and a gripper command. Policies predict 50-step chunks and execute the first ten steps before replanning.

Table~\ref{tab:real_robot_sft} summarizes the demonstration counts and SFT steps used to initialize each task's base policy.

\begin{table}[htbp]
\caption{Task-specific $\pi_{0.5}$ SFT data and training steps. Bead placement uses 20 complete demonstration trajectories.}
\label{tab:real_robot_sft}
\centering
\begin{tabular}{lrr}
\toprule
Task & Demonstrations & SFT steps \\
\midrule
Cup stacking & 25 & 30,000 \\
Toy fishing & 35 & 30,000 \\
Sequential insertion & 35 & 50,000 \\
Bead placement & 20 & 140,000 \\
\bottomrule
\end{tabular}
\end{table}

\paragraph{Task objectives.}
\emph{Cup stacking} requires placing the medium cup inside the largest cup and then the smallest cup inside the medium cup. \emph{Toy fishing} requires catching a red toy fish with a fishing rod among differently colored fish. \emph{Sequential insertion} requires inserting four pieces into their matching slots in the right-to-left order numbered in Figure~\ref{fig:real_robot_tasks}. \emph{Bead placement} requires arranging 16 colored beads into the four colored blocks of the Microsoft logo.

\paragraph{Online collection.}
Adaptation alternates between trajectory collection and model updates. Action-AWR, RLT, and autonomous NSL-AWR use robot-collected experience without human corrective demonstrations during online adaptation. DAgger uses human trajectories, while UniSteer and NSL-AWR with human guidance combine model and human data. Figure~\ref{fig:real_robot_collection} gives mean trajectory counts per round. Human counts are included in the corresponding totals; model counts equal total minus human counts.

\Needspace{0.72\textheight}
\begin{figure}[H]
\centering
\includegraphics[width=\linewidth]{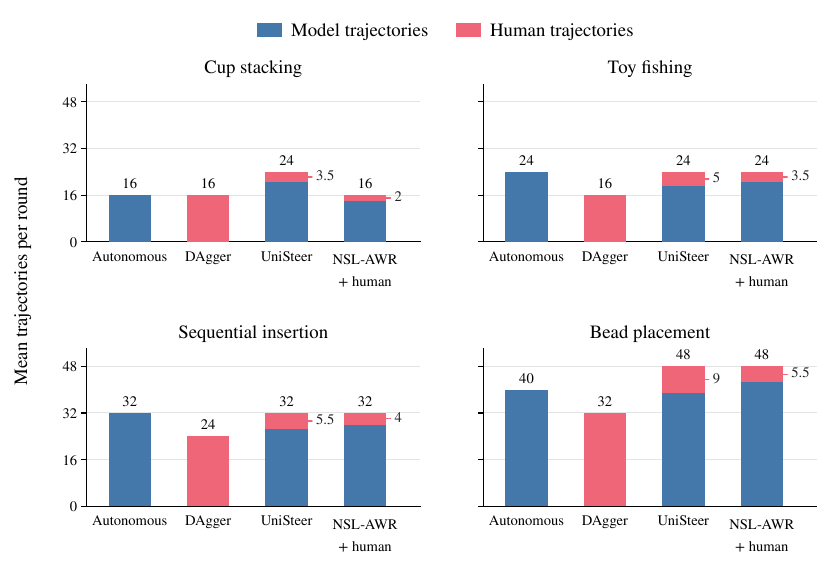}
\caption{Mean trajectories collected per adaptation round. Bars separate model and human trajectories; labels above bars give totals, and side labels give human counts. Autonomous denotes the shared schedule for Action-AWR, RLT, and NSL-AWR.}
\label{fig:real_robot_collection}
\end{figure}

\paragraph{Human-guided noise updates.}
For NSL-AWR with human guidance, we use the fixed-point inversion procedure of UniSteer~\citep{lu2026_unisteer} to recover noise corresponding to human action chunks under the frozen VLA. The recovered context--noise pairs are combined with robot-collected samples and used to update the noise actor with the AWR objective in Eq.~\eqref{eq:awr}.

\paragraph{Evaluation and timing.}
Each method is evaluated on 40 attempts per task. Task success rate is the fraction of successful attempts, and the reported average weights all four tasks equally. Training time is the measured cumulative time spent on model updates for each task, averaged over the four tasks.

\input{appendices/awr_noise_evolution}

\Needspace{0.40\textheight}
\subsection{Noise-Space SAC Diagnostics in Simulation}
\label{app:sim_noise_diagnostics}
Figure~\ref{fig:sac_noise_limit} provides diagnostics for the noise-limit ablation on SmolVLA-Full discussed in Section~\ref{sec:component_ablations}.

\begin{figure}[H]
\centering
\includegraphics[width=0.95\linewidth]{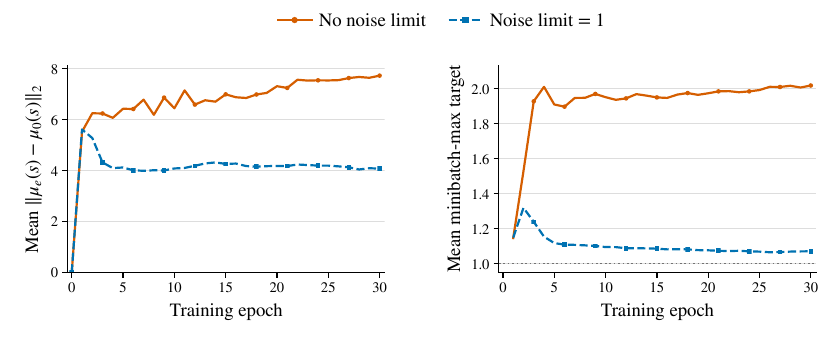}
\caption{Effect of the noise limit on NSL-SAC with SmolVLA-Full. Left: mean L2 displacement of the underlying Gaussian mean from initialization over 500 fixed contexts per run, measured before any tanh transformation. Right: mean minibatch-maximum Bellman target, weighted by the number of critic updates. The dotted line marks the return upper bound of 1 under terminal binary rewards without entropy backup.}
\label{fig:sac_noise_limit}
\end{figure}

\Needspace{0.40\textheight}
\subsection{Noise-Space SAC Diagnostics on the Real Robot}
\label{app:safety}
We analyze an online real-robot bead-placement SAC experiment that collected 32 robot trajectories, yielding 2,789 replay transitions. The policy used 32-dimensional noise transformed by $5\tanh(\cdot)$ and executed ten action steps before replanning. Unintended table contact occurred during this experiment. Figure~\ref{fig:bead_sac_diagnostics} shows 1,870 gradient updates from the recorded training segment; the horizontal axis indexes updates within online training.

\begin{figure}[H]
\centering
\includegraphics[width=\linewidth]{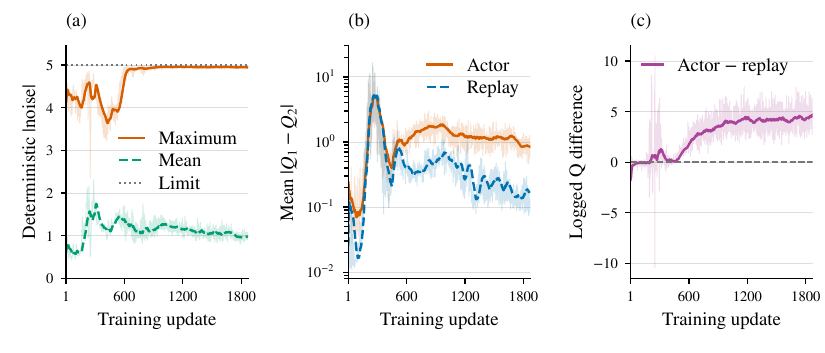}
\caption{SAC training diagnostics for real-robot bead placement. (a) Mean and maximum absolute deterministic noise output $5\tanh(\mu)$ over minibatch contexts and coordinates; the dotted line marks the amplitude limit. (b) Mean absolute twin-critic disagreement for actor-proposed and replay noise. (c) Logged actor-proposal minus replay Q value, using the arithmetic mean of the two critics for both. Thin lines show minibatch statistics; thick lines show trailing 50-update means. The Q difference compares the separately logged actor and replay evaluations.}
\label{fig:bead_sac_diagnostics}
\end{figure}

Some deterministic noise coordinates approach the amplitude limit during training. The final minibatch maximum absolute output is 4.96, while the mean absolute output increases from 0.60 to 1.04. This pattern reflects boundary proximity in a subset of outputs. Critic disagreement is also larger on actor proposals than on replay noise: over the last 200 updates, mean absolute twin-Q disagreement is 0.916 versus 0.151, a ratio of approximately 6.1. Over the same interval, the logged mean Q value for actor proposals exceeds that for replay by 4.43.

The actor proposes noises with higher predicted Q values and larger critic disagreement than replay noise, with some outputs approaching the amplitude boundary. The unintended table contact shows that limiting the noise amplitude alone does not ensure reliable decoded actions. These results motivate grounding actor updates in observed successful noise samples when robot data are limited.

%% file: appendices/awr_noise_evolution.tex
\clearpage
\subsection{NSL-AWR Noise Evolution on Fixed Observations}
\label{app:awr_noise_evolution}
Figures~\ref{fig:awr_fixed_probe_pca} and~\ref{fig:awr_fixed_probe_metrics} track the four main LIBERO NSL-AWR $1\times$ runs from epoch 0 to 60. Each run uses the same 500 fixed observations throughout the analysis, with 50 observations per task. The actors start with zero means and unit standard deviations.

\begin{figure}[H]
\centering
\includegraphics[width=\linewidth]{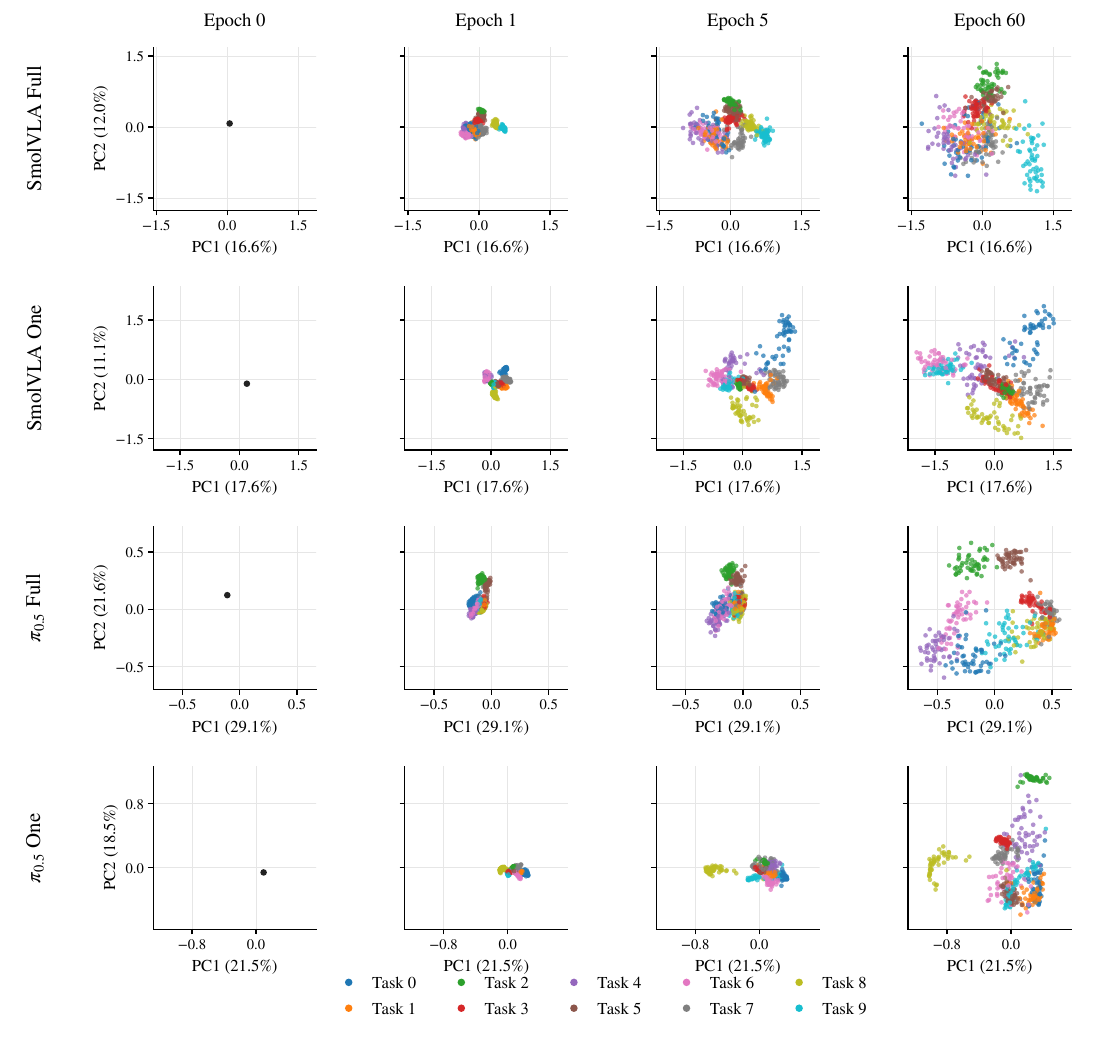}
\caption{NSL-AWR actor-mean evolution on 500 fixed LIBERO observations. Rows show the four backbone--data settings; colors denote tasks. Each point is a projected 32-dimensional conditional actor mean. Each model uses one PCA basis fitted to all 30,500 means from 61 checkpoints, with fixed axis limits across its panels. Axis percentages give the explained variance over those checkpoints. The 500 initial means overlap and are shown as one black point.}
\label{fig:awr_fixed_probe_pca}
\end{figure}

Task-related mean structure appears after the first epoch, alongside differences between observations within each task. At epoch 60, variation between task centroids accounts for 52.8--79.9\% of the total actor-mean variation in the original 32-dimensional space.

\clearpage
SmolVLA exhibits larger mean shifts than $\pi_{0.5}$ and a larger reduction in average conditional standard deviation, ending near 0.80 compared with 0.95 for $\pi_{0.5}$. For either backbone, One has a larger final mean displacement than Full, while their average standard deviations are similar.

\begin{figure}[H]
\centering
\includegraphics[width=0.80\linewidth]{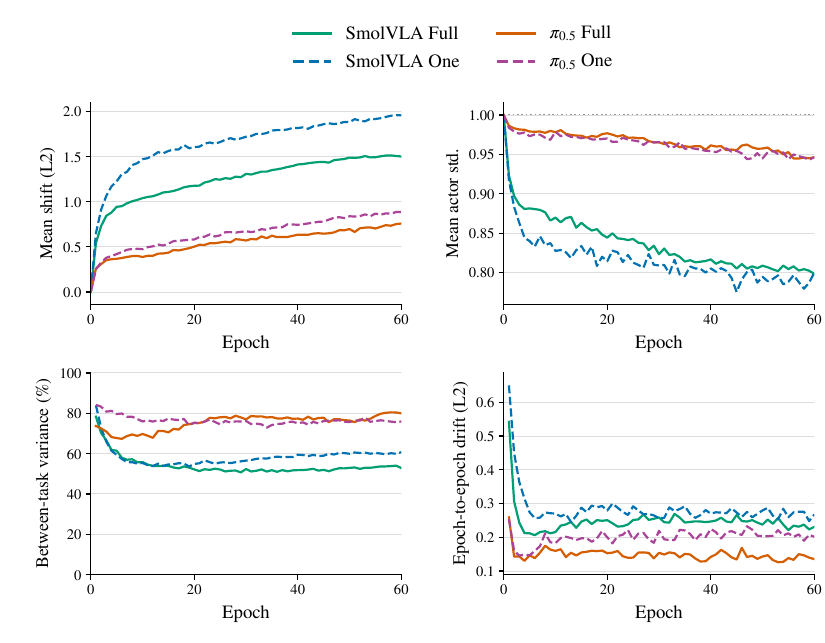}
\caption{NSL-AWR output statistics on the same fixed observations. Mean shift and epoch-to-epoch drift average per-observation L2 distances of actor means relative to initialization and the preceding checkpoint, respectively. Actor standard deviation is averaged over observations and noise coordinates. Between-task variance is the fraction of total actor-mean squared variation explained by task centroids, computed in the original 32-dimensional space. Curves use all 61 checkpoints without smoothing; the initial drift and between-task fraction are undefined and omitted.}
\label{fig:awr_fixed_probe_metrics}
\end{figure}

The outputs continue to change in later epochs. For $\pi_{0.5}$-One, the mean adjacent-epoch displacement over epochs 51--60 is 0.205, compared with 0.171 over epochs 1--5. Task-related structure thus coexists with continued refinement of observation-dependent outputs. These results illustrate how a shared noise actor develops contextual noise preferences while the action generator remains frozen.

%% file: appendices/robotwin.tex
\section{RoboTwin Experiments}
\label{app:robotwin}

\subsection{Experimental Setup}
\label{app:robotwin_setup}
We evaluate NSL-AWR, NSL-PPO, and NSL-SAC on eight RoboTwin~\citep{chen2025_robotwin} tasks under Easy and Hard settings. For each task and difficulty, fixed manifests specify 50 training
scenarios and 50 held-out evaluation scenarios with disjoint reset seeds,
including the instruction and episode horizon for each scenario.
The same requested manifests are used across methods.

Unlike our shared-policy LIBERO experiments, we train an independent noise actor for each task, without sharing RL parameters or replay across tasks.
We reuse the NSL objectives and frozen-feature interface: the VLA remains frozen, and the noise actor conditions on pooled VLA features and proprioception. 
Replay balances successful and failed experience; AWR restricts actor
regression to successful samples while fitting its value model to
both outcomes.
Noise actors start with zero Gaussian mean and log standard deviation
$-0.4$; SAC uses a $\tanh$ noise limit of $1.0$.

The main experiments run for 30 training epochs with eight rollouts and update multipliers of $4\times$, $2\times$, and $4\times$ for AWR, PPO, and SAC, respectively. Additional Easy experiments vary the update multiplier over $1\times$, $2\times$, $4\times$, and $8\times$, and the rollout count over one, four, and eight. These ablations are compared at epoch 20. We denote the update multiplier by $U$ and the rollout count by $R$.

All methods start from the same publicly released RoboTwin SmolVLA SFT
checkpoint, obtained by fine-tuning SmolVLA on 27,500 RoboTwin demonstration trajectories. The mean success rate of the SFT reference is 31.50\% on Easy and 32.75\% on Hard.

\Needspace{18\baselineskip}
\subsection{Main Results}
\label{app:robotwin_main}
Table~\ref{tab:robotwin_main} compares the three methods against SFT. All three improve both settings, with NSL-SAC achieving the highest mean success rate. At epoch 30, it reaches 84.25\% on Easy and 77.25\% on Hard, improving over SFT by 52.75 and 44.50 percentage points. NSL-AWR and NSL-PPO also provide substantial gains, with their relative ranking depending on the setting.

\begin{table}[H]
\caption{RoboTwin main results with $R=8$. SR (\%) is the equal-weight mean over all eight tasks at the same epoch. $\Delta$ is epoch-30 SR minus SFT, in percentage points. Best is the highest mean SR within epochs 1--30; parentheses give its first occurrence.}
\label{tab:robotwin_main}
\centering
\small
\begin{tabular*}{\linewidth}{@{\extracolsep{\fill}}llrrrr@{}}
\toprule
Setting & Method & SFT & Epoch 30 & $\Delta$ & Best (epoch) \\
\midrule
Easy & NSL-AWR $4\times$ & 31.50 & 73.50 & $+42.00$ & 73.50 (30) \\
     & NSL-PPO $2\times$ & 31.50 & 70.50 & $+39.00$ & 73.50 (29) \\
     & NSL-SAC $4\times$ & 31.50 & \textbf{84.25} & $+52.75$ & \textbf{84.25} (30) \\
\midrule
Hard & NSL-AWR $4\times$ & 32.75 & 65.00 & $+32.25$ & 67.00 (29) \\
     & NSL-PPO $2\times$ & 32.75 & 69.25 & $+36.50$ & 72.25 (27) \\
     & NSL-SAC $4\times$ & 32.75 & \textbf{77.25} & $+44.50$ & \textbf{78.50} (26) \\
\bottomrule
\end{tabular*}
\end{table}

\Needspace{24\baselineskip}
\subsection{Per-Task Results}
\label{app:robotwin_tasks}
Table~\ref{tab:robotwin_per_task} reports individual task results. At epoch 30, every method exceeds SFT on all eight tasks in both settings. Improvements extend to tasks with low initial success: NSL-SAC raises Easy \texttt{move\_can\_pot} from 4\% to 86\% and Hard \texttt{place\_phone\_stand} from 2\% to 62\%. The best method varies by task: NSL-AWR leads on Easy dual-bottle picking and microphone handover, while NSL-PPO leads on Hard dual-bottle picking. The task \texttt{place\_a2b\_left} remains the weakest across all six main runs.

\begin{table}[H]
\caption{Per-task SR (\%) on RoboTwin Easy and Hard at epoch 30. All runs use $R=8$, with update multipliers of $4\times$ for NSL-AWR, $2\times$ for NSL-PPO, and $4\times$ for NSL-SAC. AWR, PPO, and SAC columns denote the NSL methods. Bold marks the best entry within each difficulty block for each row, including ties.}
\label{tab:robotwin_per_task}
\centering
\small
\setlength{\tabcolsep}{2.5pt}
\begin{tabular*}{\linewidth}{@{\extracolsep{\fill}}lrrrrrrrr@{}}
\toprule
& \multicolumn{4}{c}{Easy} & \multicolumn{4}{c}{Hard} \\
\cmidrule(lr){2-5}\cmidrule(lr){6-9}
Task & SFT & AWR & PPO & SAC & SFT & AWR & PPO & SAC \\
\midrule
\texttt{handover\_block} & 10.0 & 62.0 & 62.0 & \textbf{96.0} & 10.0 & 40.0 & 56.0 & \textbf{86.0} \\
\texttt{move\_can\_pot} & 4.0 & 66.0 & 60.0 & \textbf{86.0} & 8.0 & 60.0 & 58.0 & \textbf{78.0} \\
\texttt{place\_a2b\_left} & 20.0 & 38.0 & 34.0 & \textbf{40.0} & 20.0 & \textbf{34.0} & \textbf{34.0} & 26.0 \\
\texttt{place\_phone\_stand} & 12.0 & 64.0 & 62.0 & \textbf{76.0} & 2.0 & 50.0 & \textbf{62.0} & \textbf{62.0} \\
\texttt{pick\_dual\_bottles} & 58.0 & \textbf{96.0} & 88.0 & 90.0 & 60.0 & 88.0 & \textbf{94.0} & 86.0 \\
\texttt{handover\_mic} & 66.0 & \textbf{100.0} & 96.0 & 98.0 & 58.0 & 94.0 & 90.0 & \textbf{98.0} \\
\texttt{place\_empty\_cup} & 56.0 & 94.0 & 96.0 & \textbf{100.0} & 46.0 & 80.0 & 78.0 & \textbf{92.0} \\
\texttt{beat\_block\_hammer} & 26.0 & 68.0 & 66.0 & \textbf{88.0} & 58.0 & 74.0 & 82.0 & \textbf{90.0} \\
\midrule
Mean & 31.50 & 73.50 & 70.50 & \textbf{84.25} & 32.75 & 65.00 & 69.25 & \textbf{77.25} \\
\bottomrule
\end{tabular*}
\end{table}

\Needspace{18\baselineskip}
\subsection{Analysis of Update Budget}
\label{app:robotwin_update_budget}
Table~\ref{tab:robotwin_update_budget} shows that the preferred update budget depends on the optimizer. NSL-SAC improves by 17.00 percentage points when increasing $U$ from one to four, but a further increase to eight reduces its epoch-20 SR. NSL-PPO performs best at $2\times$ and deteriorates with larger budgets; its $8\times$ run declines from 60.75\% at epoch 4 to 50.50\% at epoch 20. NSL-AWR shows comparatively stable performance across the tested update budgets, achieving its best result at $2\times$ and maintaining similar SR at $4\times$ and $8\times$. A modest update budget is therefore sufficient for strong NSL-AWR performance.

\begin{table}[H]
\caption{Update-budget ablation on RoboTwin Easy. Entries are eight-task mean SR (\%) at epoch 20 with $R=8$. Bold indicates the best multiplier for each method.}
\label{tab:robotwin_update_budget}
\centering
\small
\begin{tabular*}{\linewidth}{@{\extracolsep{\fill}}lrrrr@{}}
\toprule
Method & $U=1\times$ & $U=2\times$ & $U=4\times$ & $U=8\times$ \\
\midrule
NSL-AWR & 63.75 & \textbf{70.75} & 68.75 & 68.25 \\
NSL-PPO & 69.25 & \textbf{72.25} & 65.50 & 50.50 \\
NSL-SAC & 63.75 & 71.50 & \textbf{80.75} & 77.00 \\
\bottomrule
\end{tabular*}
\end{table}

\Needspace{18\baselineskip}
\subsection{Analysis of Rollout Count}
\label{app:robotwin_rollout_count}
Increasing the rollout count improves all three methods (Table~\ref{tab:robotwin_rollout_count}). From one to eight rollouts, epoch-20 SR increases by 7.75 percentage points for NSL-AWR, 16.75 for NSL-PPO, and 20.50 for NSL-SAC. With one rollout per scenario, AWR and SAC finish at similar performance and both exceed PPO.
Thus, SAC's advantage is strongly reduced under limited collection, while SAC benefits more from additional collection, gaining 9.75 percentage points from four to eight rollouts.
This pattern is consistent with value-guided optimization benefiting
from richer experience.

\begin{table}[H]
\caption{Rollout-count ablation on RoboTwin Easy at epoch 20. Entries are eight-task mean SR (\%). Update multipliers remain at $4\times$, $2\times$, and $4\times$ for AWR, PPO, and SAC, respectively. $\Delta$ is the increase from $R=1$ to $R=8$, in percentage points.}
\label{tab:robotwin_rollout_count}
\centering
\small
\begin{tabular*}{\linewidth}{@{\extracolsep{\fill}}lrrrr@{}}
\toprule
Method & $R=1$ & $R=4$ & $R=8$ & $\Delta$ \\
\midrule
NSL-AWR & 61.00 & 67.50 & \textbf{68.75} & $+7.75$ \\
NSL-PPO & 55.50 & 68.00 & \textbf{72.25} & $+16.75$ \\
NSL-SAC & 60.25 & 71.00 & \textbf{80.75} & $+20.50$ \\
\bottomrule
\end{tabular*}
\end{table}